\documentclass{article}

\newcommand{\EsannArXiv}[2]{#2}
\usepackage[dvips]{graphicx}
\usepackage[utf8]{inputenc}
\usepackage{amssymb,amsmath,amsthm,array}
\usepackage{hyperref}
\usepackage[capitalise]{cleveref}
\usepackage{caption}
\usepackage{subcaption}
\usepackage{algorithm}
\usepackage{algpseudocode}
\usepackage{booktabs}

\usepackage{todonotes}

\newtheorem{theorem}{Theorem}
\newtheorem{proposition}{Proposition}
\newtheorem{lemma}{Lemma}

\theoremstyle{definition}

\newcommand{\N}{\mathbb{N}}

\newcommand{\R}{\mathbb{R}}
\newcommand{\I}{\mathbb{I}}
\renewcommand{\P}{\mathbb{P}}

\newcommand{\X}{\mathcal{X}}
\newcommand{\T}{\mathcal{T}}
\newcommand{\W}{\mathcal{W}}
\newcommand{\Wm}{\mathbf{W}}
\newcommand{\D}{\mathcal{D}}
\newcommand{\PT}{P_T}
\newcommand{\TV}{\textnormal{TV}}
\newcommand{\loc}{\textnormal{loc}}
\newcommand{\Adv}{\textnormal{Adv}}
\newcommand{\1}{\mathbf{1}}
\renewcommand{\d}{\textnormal{d}}

\usepackage{pifont}
\newcommand{\no}{\ding{55}}%

\begin{document}
\title{Adversarial Attacks for Drift Detection\EsannArXiv{}{\footnote{This is an extended version. The original paper was accepted at the European Symposium on Artificial Neural Networks, Computational Intelligence and Machine Learning (ESANN).}}}

\author{Fabian Hinder, Valerie Vaquet, and Barbara Hammer\thanks{Funding from the European Research Council (ERC) under the ERC Synergy Grant Water-Futures (Grant agreement No. 951424) and funding in the scope of the BMBF project KI Akademie OWL under grant agreement No 01IS24057A 
is gratefully acknowledged.}
\vspace{.3cm}\\
Bielefeld University --
Inspiration 1, 33619 Bielefeld -- Germany
}

\maketitle

\begin{abstract}
Concept drift refers to the change of data distributions over time. While drift poses a challenge for learning models, requiring their continual adaption, it is also relevant in system monitoring to detect malfunctions, system failures, and unexpected behavior. In the latter case, the robust and reliable detection of drifts is imperative. This work studies the shortcomings of commonly used drift detection schemes. 
We show that they are prone to adversarial attacks, i.e., streams with undetected drift. In particular, we give necessary and sufficient conditions for their existence, provide methods for their construction, and demonstrate this behavior in experiments.
\end{abstract}

\section{Introduction}
\label{sec:intro}
Data from the real world is often subject to continuous changes known as concept drift
\cite{oneortwo,DBLP:journals/adt/BifetG20,asurveyonconceptdriftadaption}.
Such can be caused by seasonal changes, changed demands, aging of sensors, etc. 
Concept drift not only poses a problem for maintaining high performance in learning models~\cite{DBLP:journals/adt/BifetG20,asurveyonconceptdriftadaption} but also plays a crucial role in system monitoring~\cite{oneortwo}. 
In the latter case, the detection of concept drift is crucial as it enables the detection of anomalous behavior. Examples include machine malfunctions or failures, network security, environmental changes, and critical infrastructures. This is done by detecting irregular drifts~\cite{water_icpram,oneortwo,DBLP:journals/corr/WebbLPG17}. In these contexts, the ability to robustly detect drift is essential.

In addition to problems such as noise and sampling error, which challenge all statistical methods, drift detection faces a special kind of difficulty when the drift follows certain patterns that evade detection.
In this work, we study those specific drifts that we will refer to as ``drift adversarials''. Similar to adversarial attacks in classification, that exploit model properties to force wrong decisions~\cite{chakraborty2021survey}, drift adversarials exploit weaknesses in the detection methods, and thus allow significant concept drift to occur without triggering alarms posing major issues for monitoring systems. 
Besides the construction of drift adversarials, the presented theory also provides tools to check whether a specific drift detector is provably correct. 

This paper is structured as follows: First (\cref{sec:setup}) we recall the definition of concept drift and define the two setups for which we will construct our adversarials. In \cref{sec:main}, we construct drift adversarials that exploit which data is used for the analysis. Here, we mainly focus on two window approaches. In the last part (\cref{sec:experiments}) we perform a numerical evaluation of our considerations and conclude the work (\cref{sec:conclusion}).

\begin{algorithm}[!t]
	\caption{Two Window Drift Detector (no memory management)}
	\label{algo:dd}
	\begin{algorithmic}[1]\small
		\Procedure{DriftDetection}{$(x_i)_{i=1}^n$ data stream, $\theta$ detection threshold} \;
		\For{$(W_1,W_2) \in \W(n)$} \; \Comment{iterate over all considered window pairs}
        \State $p \gets \textsc{Test}(\{x_i \mid i \in W_1\}, \{x_i \mid i \in W_2\})$ \label{algo:dd:test}\Comment{Test samples in $W_1$ and $W_2$}
        \If{$p < \theta$}
	            \State Alert drift
	    \EndIf
		\EndFor
		\EndProcedure
	\end{algorithmic}
\end{algorithm}

\section{Concept Drift and Drift Detection}
\label{sec:setup}

Most machine learning research focuses on the batch setup where one considers a fixed data set as i.i.d. random variables $X_1,\dots,X_n$ following some distribution $\D$ on the data space $\X$. However, in many scenarios, data is obtained as a stream over time and is thus prone to potential changes of the underlying distribution, a phenomenon known as \emph{concept drift}~\cite{oneortwo,asurveyonconceptdriftadaption}. In such \emph{finite sample} setup, drift is typically defined in a \emph{sample-wise sense}, that is two samples not having the same distribution, i.e., $\exists i,j : \P_{X_i} \neq \P_{X_j}$~\cite{asurveyonconceptdriftadaption}. \emph{Drift detection} refers to the task of deciding whether or not the stream is affected by drift. 
One issue of this setup is that one cannot estimate the distribution $\P_{X_i}$ the single sample $X_i$ follows~\cite{oneortwo}. 

To analyze this task theoretically, we consider an extension building on distribution processes~\cite{oneortwo} describing the \emph{limiting case}. We model a time $\T$ indexed family of probability measures $\D_t$ on $\X$ together with an observation probability $\PT$ on $\T$~\cite{oneortwo}. 
A \emph{stream} consists of dated data points $(X_1,T_1), (X_2,T_2),\dots$ such that a data point $X_i$ observed at time $t$  follows the distribution $\D_t$, i.e., $T_i \sim \PT$ and $X_i \mid T_i = t \sim \D_t$.
\emph{Concept drift} occurs if the chance of observing two different distributions is larger zero~\cite{oneortwo}, i.e., $\P[\exists i,j : \P_{X_i} \neq \P_{X_j}] > 0$. 
Notice that this definition is not limited to abrupt drift but all kinds of drift, such as gradual or recurring, and the statistical nature resolves the estimation problem. 

In this paper, we are interested in constructing scenarios containing drift that is not detected. We will first investigate this task theoretically by examining the limiting case leveraging the definition by distribution processes. This also allows us to prove guarantees. Afterward, we study the finite case and derive a practical algorithm for the construction of drift adversarials.

\section{Adversarial Attacks for Drift Detection}
\label{sec:main}

Most drift detectors process data on sliding windows using some statistical tool, most commonly a metric~\cite{oneortwo} (see \cref{algo:dd}). This allows for two natural attack scenarios: \emph{Metric Adversarials} construct distributions indistinguishable by the metric, while \emph{Window Adversarials} exploit the data selection stage. 

\subsection{Metric Adversarials}
The most commonly used drift detectors are based on learning models referring to the optimal model or model accuracy to detect drift~\cite{asurveyonconceptdriftadaption}. However, as already pointed out in \cite{icpram} this approach is flawed and can be exploited in many cases. Indeed, the authors provide a constructive proof that can easily be modified to construct a metric adversarial. Other approaches for which metric adversarials can be constructed include methods like the windowing Kolmogorov-Smirnov test that operates feature-wise and thus ignores drifts in correlations as shown in~\cite{oneortwo} or methods that use deep embeddings for which classical adversarials can be constructed. However, in many cases, it is not possible to construct a metric adversarial, e.g., when the used metric is indeed a metric. We will therefore mainly focus on window adversarials in this paper. 

\subsection{Window Adversarials for Two-Window-Based Detectors}
As most drift detectors work by comparing data from two windows \cite{oneortwo} we will focus on this setup. In the following, we consider the limiting and finite case.

\paragraph{The limiting case} 
We refer to the case where we take the sampling rate to infinity so that errors due to sampling vanish and the drift detector becomes a map of the kernel $\D_t$, i.e., the limiting case of \cref{algo:dd:test} takes on the form
\begin{align}
    A(\D_t) = \1\left[ \sup_{(W_1,W_2) \in \W} d(\D_{W_1},\D_{W_2}) > 0\right] \label{eq:limit_DD}
\end{align}
\begin{table}[t]
    \centering
    \caption{Overview of improper adversarial functions for common windowing schemes used in drift detection (assuming Lebesgue measure $\PT=\lambda$). Cases with Boundary Effects (BE) are marked. Proofs in \EsannArXiv{ArXiv version~\cite{arxivVersion}}{appendix}.
    \label{tab:adv}}
    {\small
    \begin{tabular}{cccc}
        \toprule
        $\T$ & $\W$ & $\Adv_0$  (\cref{thm:main}) & BE \\
        \midrule
        $\R$ & $([t-l,t],[t,t+l])$ $t \in \T$ & $f(t) = f(t+l)$ & \no \\
        & \footnotesize two sliding windows & \footnotesize $l$-periodic functions \\[0.4em]
        $\R_{\geq 0}$ & $([0,a],[t,t+l])$ $t \geq a$ & $\text{$f(t) = f(t+l)$ for $t \geq a$ and}\atop\text{$a^{-1} \int_0^a f(t) \d t = l^{-1}\int_a^{a+l} f(t) \d t$}$ &  \\
        & \footnotesize fixed reference window & \footnotesize $l$-periodic after $a$ with same mean \\[0.4em]
        $\R_{\geq 0}$ & $([0,t],[t,t+l])$ $t \geq a$ & $a^{-1} \int_0^a f(s) \d s = f(a) = f(t) \forall t \geq a$ & \no \\
        & \footnotesize growing reference window & \footnotesize arbitrary before $a$ and then constant \\
        \bottomrule
    \end{tabular}
    }
\end{table}
where $\W$ is the set of all window pairs directly compared by the detector and $d$ is the used metric. Here, $A$ detects drift if $A(\D_t) = 1$. As $A$ cannot have false positives, the window adversarials are given by false negatives which can be constructed as follows:
\begin{theorem}
\label{thm:main}
Define the \emph{improper adversarial functions} for $A$ as in \cref{eq:limit_DD} as\vspace{-0.8em}\linebreak\resizebox{\textwidth}{!}{%
\begin{minipage}{1.2\textwidth}%
\begin{align}\Adv_0(A) = \left\{f : \T \to [0,1] \:\left|\: \PT(W_2)\int_{W_1} \!\!\!\!f\,\, \d \PT = \PT(W_1)\int_{W_2} \!\!\!\!f\,\, \d \PT \forall (W_1,W_2) \in \W\right.\right\}\end{align}\end{minipage}}\linebreak 
then $A$ detects no drift, i.e., $A(\D_t) = 0$, if and only if $t \mapsto \D_t(S) \in \Adv_0(A)$ for all measurable $S \subset \X$.

Define the \emph{adversarial functions} $\Adv(A) \subset \Adv_0(A)$ as those that are not constant. 
Then, $\Adv(A)$ describes all distribution processes with drift that is not detected by $A$. 
In particular, for $f \in \Adv(A)$ and distributions $P \neq Q$ on $\X$, $\D_t = f(t)P + (1-f(t))Q$ is a window adversarial, i.e., $\D_t$ has drift and $A(\D_t) = 0$. Conversely, for every window adversarial $A(\D_t) = 0$ we have $t \mapsto \D_t(S) \in \Adv_0(A)$. 
Therefore, if $\Adv(A) = \emptyset$ then $A$ detects every drift assuming $d$ is a metric. 
\end{theorem}
\EsannArXiv{
\begin{proof}
    All proofs can be found in the ArXiv version~\cite{arxivVersion}.
\end{proof}}{The proof is given in the appendix.}

We want to stress that the adversaials do not depend on the metric $d$ but only the considered windows and that the drift detector detects every possible drift if and only if $\Adv(A) = \emptyset$. This can for example be achieved by combining multiple drift detectors as $\Adv( \,(A,B)\,) \subset \Adv(A) \cap \Adv(B)$.

Most detectors use a sliding window for the current distribution of fixed length. There are three main strategies for the reference window: 1)~fixed, 2)~growing, and 3)~sliding with fixed length~\cite{oneortwo}. Furthermore, there are two update strategies: Either the update is performed after every single data point, which in the limit is for every time point, or by considering chunks of data points. For the latter, we can hide arbitrary drifts within a chunk allowing for trivial adversarials. For point-wise updates and any of the aforementioned reference windows, we present the adversarial functions in \cref{tab:adv}. 

\paragraph{The finite case} Analog to the limiting case we can also consider the case of finite samples $X_1,\dots,X_n$.
In this case, the windows refer to which samples are considered, i.e., $W_1,W_2 \subset [n]$. We will denote the set of all window pairs for $n$ samples by $\W(n)$ together with a normalized distance measure, e.g., a statistic test, and a decision threshold $\theta$ this leads to \cref{algo:dd}. Usually, there is some memory management so that we do not have to store the entire stream which however depends on the windowing scheme $\W(n)$.

\begin{algorithm}[!t]
	\caption{Construction of Drift Adversarials}
	\label{algo:adv}
	\begin{algorithmic}[1]\small
		\Function{ConstructDriftAdversarial}{$P,Q$ sampling distributions, $\W(n)$ windowing scheme to be attacked} \;
        \State $\Wm_n \gets [\textbf{1}]$
		\For{$(W_1,W_2) \in \W(n)$} \;
		\State $\Wm_n \gets \Wm_n + \left[ |W_{1}|^{-1} \sum_{i = 1}^n \1[i \in W_{1}] e_i - |W_{2}|^{-1} \sum_{i = 1}^n \1[i \in W_{2}] e_i \right]$
        \EndFor
        \State $v \gets \textsc{Solve}(\Wm_n x = 0)$ \label{algo:adv:solve} \Comment{Interpret $\Wm_n$ as a matrix}
        \State $v \gets \frac{v - \min_i v_i}{\max_i v_i - \min_i v_i}$
        \State $x \gets []$
        \For{$i = 1,\dots,n$}
        \State $x \gets x + [\textsc{Sample}(v_iP + (1-v_i)Q)]$
        \EndFor 
        \State\Return $x$
        \EndFunction
	\end{algorithmic}
\end{algorithm}

We can encode the window selection $\W(n)$ into a single weight matrix $\Wm_n$ that encodes the pair $(W_1,W_2)$ as the vector $w = |W_1|^{-1}\sum_{i \in W_1} e_i - $\linebreak$|W_2|^{-1}\sum_{i \in W_2} e_i$ where $e_i$ is the $i$-th coordinate vector. This representation is quite useful as it for example allows us to write the biased MMD -- a kernel-based probability metric commonly used in drift detection~\cite{oneortwo} -- of the $i$-th window as $(\Wm_n^\top K \Wm_n)_{ii}$ where $K_{ij} = k(X_i,X_j)$ is the kernel matrix. For our purpose, it is useful as the kernel of $\Wm_n$ can be used to construct drift adversarials. To do so choose $v \in [0,1]^n$ with $\Wm_n v = 0$ and then sample $X_i \sim v_iP + (1-v_i)Q$. 
This idea is represented in \cref{algo:adv}. If $v$ is not constant and $P \neq Q$ then the distributions differ for some $X_i$, i.e., there is drift in the sample-wise sense, the mean distributions of the samples in $W_1$ and $W_2$ however coincide for all $(W_1,W_2) \in \W(n)$ which is what \cref{algo:dd} line~\ref{algo:dd:test} is testing for. There are ways to increase the quality by choosing $v \in \{0,1\}^n$ or trying to avoid fast oscillations as such streams are similar to non-drifting streams.
Notice, that there is a tight connection between the limiting and the finite setup which is given by sampling adversarial functions (\cref{thm:main}) equidistant to obtain $v$. 
Yet, $\{v \mid \Wm_n v = 0\}$ can be much larger than $\Adv(A)$ due to boundary effects (BE in \cref{tab:adv}).

Instead of comparing the mean distribution of two windows, some drift detectors -- dubbed block-based in \cite{oneortwo} -- check for any kind of drift within a single window. Using similar techniques it can be shown that such detectors are not prone to window adversarial attacks. 
In the next section, we will test our theoretical observations empirically.

\section{Empirical Evaluation}
\label{sec:experiments}

To evaluate our methodology we consider two empirical setups: a numerical analysis on synthetic data, and a showcase on data from critical infrastructure.\footnote{The code can be found at \url{https://github.com/FabianHinder/Drift-Adversarials}}

\begin{table}[t]
  \caption{Result of numerical analysis. $90\%/10\%$-quintile of obtained $p$-values (500 runs). Correct result is $p = 0$, lining marks adversarials according to theory. The number in brackets is the length of the initial reference window.}
      \centering
      {\footnotesize
  \begin{tabular}{lc@{\,}c@{\,}c@{\,}c@{\,}c@{\,}c}
\toprule
Dataset / $\W$ & fixed (100) & fixed (150) & grow (100) & grow (150) & slideing \\
\midrule
Periodic & $\underline{0.63/0.28}$ & $0.00/0.00$ & $0.00/0.00$ & $0.00/0.00$ & $\underline{0.39/0.27}$ \\
Rand.Const (100) & $\underline{0.59/0.31}$ & $\underline{0.54/0.31}$ & $\underline{0.39/0.27}$ & $\underline{0.46/0.28}$ & $0.00/0.00$ \\
Rand.Const (150) & $0.00/0.00$ & $\underline{0.50/0.25}$ & $0.00/0.00$ & $\underline{0.45/0.24}$ & $0.00/0.00$ \\
Rand.Per. (100) & $\underline{0.49/0.26}$ & $0.00/0.00$ & $0.00/0.00$ & $0.00/0.00$ & $0.00/0.00$ \\
Rand.Per. (150) & $0.00/0.00$ & $\underline{0.65/0.22}$ & $0.00/0.00$ & $0.02/0.00$ & $0.00/0.00$ \\
\bottomrule
\end{tabular}}
  \label{tab:evaluation}
\end{table}
\paragraph{Synthetic Data} We perform a numerical analysis based on the simple two-squares dataset \cite{oneortwo} (drift intensity 5). We create the adversarial streams using \cref{algo:adv} where line~\ref{algo:adv:solve} is performed by hand to assure $v_i \in \{0,1\}$ with as little changes as possible (see \cref{tab:adv}). Each stream has a length of $1{,}000$ samples, window sizes of the sliding window is 100, and (initial) reference window is 100/150. We use the permutation MMD test~\cite{oneortwo} with $2{,}500$ permutations and consider the smallest $p$-value found in the stream. We performed $500$ independent runs for each setup. The results are reported in \cref{tab:evaluation}. 
As can be seen, there is a (nearly) perfect alignment of our theoretical predictions and the empirical results with only one exceptional case.

\paragraph{Application to Water Distribution Networks} 
Thus far we have considered drift adversarials as a kind of attack where we try to hide the drift from the monitoring system. However, in case we expect certain drifts that we do not want to detect, we can try to construct a drift detector so that the adversarials are exactly the expected drifts. For this study, we explicitly consider water distribution networks from which we obtain pressure measurements~\cite{water_icpram}. We are interested in leakage detection which can be done via drift detection. However, as the demands on drinking water are not constant over time, we expect daily (day-night-cycle) and weekly (week-weekend-cycle) patterns, which need to be removed before drift and leakages are directly related~\cite{water_icpram}. 
Following~\cite{water_icpram}, we use the Shape Drift Detector~\cite{hinder2021shape} which postprocesses the MMD of two consecutive sliding windows to find candidate drift points. The result of different window lengths is presented in \cref{fig:shape}. As can be seen, windows of one-day length detects weekends, while a one-week length window mainly detects the leakage as desired. Also notice, that this is not an instability of the algorithm as can be seen by considering the window length of $6\frac{1}{2}$ days (middle figure). 
\begin{figure}
    \centering
    \includegraphics[width=0.33\textwidth,height=1.5cm,trim={5mm 6mm {380mm} 7mm},clip]{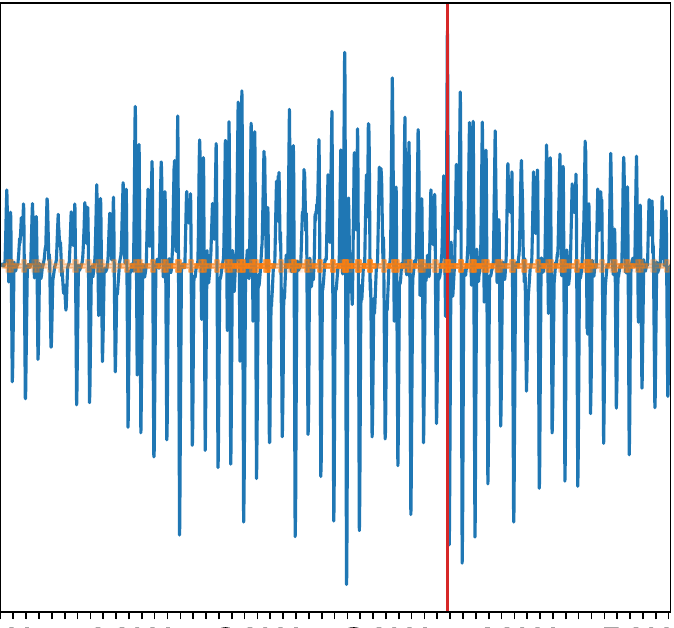}
    \!\includegraphics[width=0.33\textwidth,height=1.5cm,trim={{131mm} 6mm {254mm} 7mm},clip]{shape.pdf} 
    \!\includegraphics[width=0.33\textwidth,height=1.5cm,trim={{257mm} 6mm {128mm} 7mm},clip]{shape.pdf}
    \caption{Shape curve for different window sizes (1 day, $6\frac{1}{2}$ days, 1 week). Red line marks leakage, orange crosses candidate points (transparency is MMD).}
  \label{fig:shape}
\end{figure}

\section{Conclusion}
\label{sec:conclusion}
In this paper, we considered the concept of drift adversarials. We showed that many commonly used drift detectors are subject to at least some drift adversarial attacks. We considered the problem from a general theoretical and concrete point of view and evaluated our findings empirically. Furthermore, we investigated the potential of our theory to construct problem-tailored drift detectors which seems to be a promising approach but requires further research. 

Our considerations show that drift adversarials pose a major problem but might be numerically unstable. A further analysis is yet subject to future work. 

\begin{footnotesize}
\bibliographystyle{unsrt}
\bibliography{bib}
\end{footnotesize}

\EsannArXiv{}{
\newpage\appendix

\section{Proofs}
In the following section, we will provide formal proofs of the statements made above.

\subsection{Proof of \cref{thm:main}}
We start with the proof our the main theorem:
\begin{proof}[Proof of \cref{thm:main}]
We have $d(\D_{W_1},\D_{W_2}) = 0$ if and only if $\D_{W_1} = \D_{W_2}$ if and only if 
\begin{align*}
0 &= \Vert \D_{W_1} - \D_{W_2} \Vert_\TV 
\\&= \sup_{S \in \Sigma_\X} \left|\D_{W_1}(S) - \D_{W_2}(S)\right| 
\\&=  \sup_{S \in \Sigma_\X} \left|\PT(W_1)^{-1}\int_{W_1} \D_t(S) \d \PT(t) - \PT(W_2)^{-1}\int \D_t(S) \d \PT(t)\right|
\\&= \frac{\sup_{S \in \Sigma_\X} \left|\PT(W_2)\int_{W_1} \D_t(S) \d \PT(t) - \PT(W_1) \int_{W_2} \D_t(S) \d \PT(t)\right|}{\PT(W_1)\PT(W_2)}
\end{align*}
and therefore
\begin{align*}
0 &= \sup_{(W_1,W_2) \in \W} d(\D_{W_1},\D_{W_2})\\\Leftrightarrow 0 &= \sup_{(W_1,W_2) \in \W} \Vert \D_{W_1} - \D_{W_2} \Vert_\TV 
\\\Leftrightarrow 0 &= \sup_{(W_1,W_2) \in \W}\sup_{S \in \Sigma_\X} \left|\PT(W_2)\int_{W_1} \D_t(S) \d \PT(t) - \PT(W_1) \int_{W_2} \D_t(S) \d \PT(t)\right|
\\&= \sup_{S \in \Sigma_\X}\sup_{(W_1,W_2) \in \W} \left|\PT(W_2)\int_{W_1} \D_t(S) \d \PT(t) - \PT(W_1) \int_{W_2} \D_t(S) \d \PT(t)\right|
\end{align*}
which is the case if and only if $t \mapsto \D_t(S) \in \Adv_0(A)$.

Obviously $\Adv_0(A)$ is convex. So in particular, if $f \in \Adv_0(A)$ then $\lambda f + \mu \in \Adv_0(A)$ if the resulting function stays in $[0,1]$. Hence for any $P,Q \in \textnormal{Pr}(\X)$ and $S \in \Sigma_\X$ we have $f(t)P(S) + (1-f(t))Q(S) = f(t)(P(S)-Q(S)) + Q(S) \in \Adv_0(A)$. 

Conversely, $\D_t$ has drift if and only if $t \mapsto \D_t(S) \not\in \text{Const}(\T)$ for some $S \in \Sigma_\X$. As $A$ detects the drift if and only if $t \mapsto \D_t(S) \not\in \Adv_0(A)$ for some $S \in \Sigma_\X$ and $\Adv_0(A) = \text{Const}(\T) \cup \Adv(A)$ we see that the statement ``$A$ detects drift if and only if $\D_t$ has drift'' holds true if and only if $\Adv(A) = \emptyset$. 
\end{proof}

\subsection{Proofs of the statements in \cref{tab:adv}}
We will now compute the improper adversarial sets displayed in \cref{tab:adv}. To do so, we will compute the sets 
\begin{align*}
\ker \W = \left\{f \in L^1_\loc(\T) \:\left|\: \lambda(W_2)\int_{W_1} f(t) \d t = \lambda(W_1)\int_{W_2} f(t) \d t \forall (W_1,W_2) \in \W\right.\right\}.
\end{align*}

\begin{proposition}[Fixed reference]
    For $\W = \{ ([0,a],[t,t+l]) \mid a \leq t \}$ with $a, l > 0$ and $f \in L^1_\loc(\R_{\geq 0})$ the following are equivalent 
    \begin{enumerate}
        \item $f \in \ker \W$
        \item for all $t > a$ we have
        \begin{align*}
            \frac{l}{a} \int_0^a f(x) \d x = \int_{t}^{t+l}  f(x) \d x 
        \end{align*}
        \item $f$ is $l$-periodic after $a$, i.e., $f(t) = f(t+l) \forall t > a$, and has the same mean as before, i.e., $\frac{l}{a} \int_0^a f(x) \d x = \int_a^{a+l} f(x) \d x$.
    \end{enumerate}
    In particular, $\Adv_0(\W)\R = \ker \W$ are the functions that are $l$-periodic after $a$ and have the same mean as before. There are no boundary effects. 
\end{proposition}
\begin{proof}
    \emph{1. $\Leftrightarrow$ 2.} is easily seen by rewriting.

    \emph{2. $\Rightarrow$ 3.} define $C \in \R$ and $F : \R_{\geq 0} \times \N \to \R$ as
    \begin{align*}
        C &:= \frac{l}{a}\int_0^a f(x) \d x & \text{and} \\ 
        F_t(n) &:= \int_0^{a+t+ln} f(x) \d x.
    \end{align*}
    By definition of $C,F_t$ and by assumption we have $C = F_t(n+1)-F_t(n)$ for all $t \geq 0$ and $n \in \N$. Considering this as a recurrent equation in $n$ for every single $t$ it thus follows 
    \begin{align*}
        F_t(n) = c(t) + (t/l+n) \cdot C
    \end{align*}
    where $c(t)$ is the function parameter for each $t$ which is uniquely determined by the equation. 
    
    On the other hand, since $F_{t+l}(n) = F_t(n+1)$ and $(t+l)/l+n = t/l + (n+1)$ we have $c(t) = c(t+l)$ so $c$ is $l$-periodic. Therefore, the difference 
    \begin{align*}
        F_{t+h}(0)-F_t(0) &= (c(t+h)-c(t)) + h/l C \\&= (c(t+h+l) - c(t+l)) + h/l C \\&= F_{t+h+l}(0) - F_{t+l}(0)
    \end{align*}
    is  $l$-periodic, too. On the other hand by Lebesgue's differentiation theorem $(F_{t+h}(0)-F_{t-h}(0))/(2h) \to f(t)$ as $h \to 0$ for almost every $t$. Therefore, $f$ is $l$-periodic as well for all $t \geq a$.
    Furthermore, we can conclude that
    \begin{align*}
        \int_{a}^{a+l} f(x) \d x = \frac{l}{a}\int_0^a f(x) \d x.
    \end{align*}

    \emph{3. $\Rightarrow$ 2.} if $f$ is $l$-periodic after $a$ we have $\int_{t}^{t+l} f(x) \d x = c$ is $t$-invariant for all $t > a$ and if it has the same mean it holds
    \begin{align*}
        \frac{l}{a}\int_0^a f(x) \d x = \int_{a}^{a+l} f(x) \d x = \int_{t}^{t+l} f(x) \d x.
    \end{align*}
    Therefore, the statement follows.
\end{proof}

To prove the next statement we need the following simple lemma:
\begin{lemma}
    \label{lem:limit}
    Let $a_n, b_n \subset \R$ and $c_1,c_2 \in \R$ with $c_1 \neq c_2$. The following are equivalent
    \begin{enumerate}
        \item The limits $a_n$ and $b_n$ as $n \to \infty$ exist
        \item The limits $a_n + c_1 b_n$ and $a_n +c_2b_n$ exist as $n \to \infty$
    \end{enumerate}
\end{lemma}
\begin{proof}
    \emph{1. $\Rightarrow$ 2.} is just linearity. 

    \emph{2. $\Rightarrow$ 1.} follows by considering $(a_n - c_1b_n) - (a_n - c_2 b_n) = (c_1-c_2) b_n$. As the limit $n \to \infty$ on the left hand side exists so does the limit on the right hand side which by linearity implies that the limit of $b_n$ exists. Thus, by linearity, the limit of $a_n$ exists, too. 
\end{proof}

\begin{proposition}[Sliding windows]
    For $\W = \{ ([t-l,t],[t,t+l]) \mid t \in \R \}$ with $l > 0$ and $f \in L^1_\loc(\R)$ the following are equivalent 
    \begin{enumerate}
        \item $f \in \ker \W$
        \item for all $t \in \R$ we have
        \begin{align*}
            \int_{t-l}^t f(x) \d x = \int_t^{t+l} f(x) \d x
        \end{align*}
        \item $f(x) = p(x) + tq(x)$ with $p$ and $q$ $l$-periodic and $\int_0^l q(x) \d x = 0$.
    \end{enumerate}
    In particular, the solution has boundary effects, and $\Adv_0(\W)\R \subsetneq \ker \W$ are the $l$-periodic functions. 
\end{proposition}
\begin{proof}
    \emph{1. $\Leftrightarrow$ 2.} is easily seen by rewriting.

    \emph{2. $\Rightarrow$ $f(t) = q(t) + tq(t)$ with $p,q$ $l$-periodic:} 
    defining $F_t(n) = \int_0^{t+nl} f(x) \d x$ induces the a recurrent equation for every $t$ 
    \begin{align*}
        F_t(n-1) + F_t(n+1) - 2F_t(n) = 0
    \end{align*}
    which has the solution $F_t(n) = c_1(t) + c_2(t)(t/l+n)$ with parameter functions $c_1,c_2$. 
    
    As $F_{t+ml}(n) = F_{t}(n+m)$ we have 
    \begin{align*}
        0 &= F_{t+ml}(n) - F_{t}(n+m) \\&= (c_1(t+ml)-c_1(t)) + (c_2(t+ml)-c_2(t))(t/l+n+m)
    \end{align*}
    and therefore 
    \begin{align*}
        0 &= 0 - 0 \\&= (F_{t+ml}(n) - F_{t}(n+m)) - (F_{t+ml}(n+1) - F_{t}(n+m+1)) \\&= (c_2(t+ml)+c_2(t))
    \end{align*}
    and hence $c_2$ is $l$-periodic but this then also implies that 
    \begin{align*}
        0 = (c_1(t+ml)-c_1(t)) + \underbrace{(c_2(t+ml)-c_2(t))}_{=0}(t/l+n+m)
    \end{align*}
    so $c_1$ is $l$-periodic, too. 

    Furthermore, as $t \mapsto F_t(n)$ is continuous and we have 
    \begin{align*}
        |c_2(t+h)-c_2(t)|
        &=|(F_{t+h}(0)-F_t(0))-(F_{t+h}(1)-F_t(1))|
        \\&\leq|F_{t+h}(0)-F_t(0)|+|F_{t+h}(1)-F_t(1)| \xrightarrow{h \to 0}0
    \end{align*}
    we conclude that $c_2$ and hence $c_1$ are continues, too.
    
    Now considering 
    \begin{align*}
        \frac{F_{t+h}(0) - F_{t-h}(0)}{2h} 
          &=\quad \frac{c_1(t+h)-c_1(t-h)}{2h} 
        \\&\quad+ \frac{c_2(t+h)-c_2(t-h)}{2h} \cdot (t+h) 
        \\&\quad+ c_2(t+h) \cdot \frac{(t+h)-(t-h)}{2h}
        \\&=\quad \underbrace{\frac{c_1(t+h)-c_1(t-h)}{2h} + c_2(t+h)}_{=:P(t)} 
        \\&\quad+ t \cdot \underbrace{\frac{c_2(t+h)-c_2(t-h)}{2h}}_{=: Q(t)} 
        \\&\quad+ \underbrace{\frac{c_2(t+h)-c_2(t-h)}{2}}_{\in o(h)}
    \end{align*}
    we see the quotient is of the form $P(t)+t Q(t) + o(h)$ for all $h > 0$ with $P$ and $Q$ $l$-periodic since $c_1$ and $c_2$ are $l$-periodic.
    On the one hand, have $(F_{t+h}(0) - F_{t-h}(0))/(2h) \to f(t)$ as $h \to 0$ by Lebesgue's differentiation theorem. As $P$ and $Q$ are $l$-periodic we obtain the same sequences for $a_n := P(t+h_n) = P(t+h_n+l)$ and $b_n := Q(t+h_n) = Q(t+h_n+l)$ so by \cref{lem:limit} with $c_1 = t$ and $c_2 = t+l$ we see that the (point-wise) limits exist and thus obtain $l$-periodic $L^1_\loc$-function $p$ and $q$ with $f(t) = p(t) + q(t)t$.
    
    As $f(t) \in [0,1]$ it follows that $q(t) = 0$ as otherwise $q(t)t > 1$ for $|t|$ sufficiently large. In particular, the solution has boundary effects. 
    
    \emph{2. $\Leftrightarrow$ 3.} if 3. holds or -- according to the previous claim -- if 2. holds, we have $f(t) = p(t) + tq(t)$ with $p,q \in L^1_\loc(\R)$ $l$-periodic. Therefore, it remains to show that $3. \Leftrightarrow \int_0^l q(x) \d x = 0$. It holds
    \begin{align*}
             \int_{t-l}^t f(x) \d x - \int_{t}^{t+l} f(x) \d x 
          &= \int_{t-l}^t p(x) + xq(x) \d x - \int_{t}^{t+1} p(x) + xq(x) \d x
        \\&= \int_{t-l}^t p(x) + xq(x) \d x - \int_{t-l}^{t} \underbrace{p(x-l)}_{=p(x)} + (x-l)\underbrace{q(x-l)}_{=q(x)} \d x
        \\&= \int_{t-l}^t ((x)-(x-l))q(x) \d x
        \\&= l \int_{t-l}^t q(x) \d x
    \end{align*}
    and as $q$ is $l$-periodic we have $\int_{t-l}^t q(x) \d x = \int_{0}^l q(x) \d x$. Therefore $f \in \ker \W$ if and only if $\int_{0}^l q(x) \d x = 0$.
\end{proof}

\begin{proposition}[Growing reference]
    Let $\W = \{ ([0,t],[t,t+l]) \mid t \geq a \}$ with $a,l > 0$ and $f \in L^1_\loc(\R_{\geq 0})$. The following are equivalent:
    \begin{enumerate}
        \item $f \in \ker \W$
        \item for all $t > a$ we have
        \begin{align*}
            \frac{l}{t}\int_{0}^t f(x) \d x = \int_t^{t+l} f(x) \d x
        \end{align*}
        \item $f(t) = \I_{[0,a]}(t)g(t) + p(t) + tq(t)$ for all $t > 0$ with $p$ and $q$ $l$-periodic, $p(t) = \int_0^t q(x) \d x + C$, $\int_0^l q(x) \d x = 0$, and $\int_0^a g(x) \d x = 0$.
    \end{enumerate}
    In particular, there are boundary effects, i.e., $\Adv_0(\W)\R \subsetneq \ker \W$, and $\Adv_0(\W)\R$ are exactly the functions that are arbitrary before $a$ and then constant with the same mean as before, i.e., $a^{-1}\int_0^a f(x) \d x = f(t)$ for all $t > a$.
\end{proposition}
\begin{proof}
    \emph{1. $\Leftrightarrow$ 2.} is easily seen by rewriting.

    \emph{2. $\Rightarrow$ Properties of 3. except $\int_0^a g(x) \d x = 0$} define $F_t(n) = \int_0^{t+ln} f(x) \d x$ then we have $$(t+l(n-1))^{-1} F_t(n-1) = l^{-1}(F_t(n)-F_t(n-1))$$ for all $t \geq a$ which allows the recursive definition for each $t > a$
\begin{align*}
F_t(n) &= \frac{l + l(n-1) + t}{l(n-1) + t} F_t(n-1) \\
    \\&=  \underbrace{\frac{ln + t}{l(n-1) + t} \frac{l(n-1) + t}{l(n-2) + t}}_{= \frac{ln +t}{l (n-2) + t}}  F_t(n-2) 
    \\&= \cdots = \frac{ln + t}{t} F_t(0)
    \\&= (t+ln)\underbrace{\frac{1}{t}F_t(0)}_{=:c(t)}
\end{align*}
As before we conclude from $F_{t+l}(n) = F_{t}(n+1)$ that 
\begin{align*}
    0 = F_{t}(n+1) - F_{t+l}(n) = (t+l(n+1))(c(t)-c(t+l))
\end{align*}
so $c$ is $l$-periodic. And from
\begin{align*}
    (F_{t+h}(1)-F_t(1))-(F_{t+h}(0)-F_t(0)) &= 
    c(t+h)-c(t)
\end{align*}
that since $t \mapsto F_t(n)$ is continuous that $c(t)$ is continuous. We extend $c$ to a function on all of $\R$ by using its periodicity.

Considering the quotient for all $t> a, \, 0 < h < t-a$ 
\begin{align*}
        \frac{F_{t+h}(0)-F_{t-h}(0)}{2h} 
      &=\left(t-h\right)\frac{c(t+h)-c(t-h)}{2h} + c(t+h).
\end{align*}
Thus, by Lebesgue's differentiation theorem, we then conclude that $f(t) = \I_{[0,a]}(t)g(t) + p(t) + tq(t)$ with $p,q$ $l$-periodic, and $p(t) = C + \int_0^t q(x) \d x$ for all $t > 0$ and $g$ arbitrary. In particular, as $p$ is $l$-periodic we have $0 = p(l)-p(0) = \int_0^l q(x) \d x$. 

\emph{2. $\Leftrightarrow$ 3.} due to the relations, we have
\begin{align*}
    Q(t) &= \int_0^t q(x) \d x = p(t)-C, & \text{and}\\
    \int_0^t Q(x) \d x &= \int_0^t p(x) \d x - tC.
\end{align*}
Denote by $C_0 = \int_0^a g(x) \d x$, then we have
\begin{align*}
         \frac{1}{t}\int_0^t f(x) \d x 
      &= \frac{1}{t}\left(\int_0^a g(x) \d x + \int_0^t p(x) + xq(x) \d x\right)
    \\&= \frac{1}{t}\left(C_0 + \int_0^t p(x) \d x + \left[Q(x)x\right]_0^t - \int_0^t Q(x) \d x\right)
    \\&= \frac{1}{t}\left(C_0 + tC + Q(t)t \right) 
    \\&= \frac{C_0}{t} + C + Q(t)\\
         \frac{1}{l}\int_t^{t+l} f(x) \d x
      &= \frac{1}{l}\left(\int_t^{t+l} Q(x) \d x + lC + \left[Q(x)x\right]_t^{t+l} - \int_t^{t+l} Q(x) \d x\right)
    \\&= \frac{1}{l}\left(lC + Q(t)l\right)
    \\&= C + Q(t), \\
    \Rightarrow  \frac{C_0}{t}  &= \frac{1}{t}\int_0^t f(x) \d x-\frac{1}{l}\int_t^{t+l} f(x) \d x 
\end{align*}
we conclude that $f \in \ker \W$ if and only if $C_0 = 0$.

As before we conclude that form $|f(t)| \leq 1$ that $q(t) = 0$ and hence $p(t) = C$ is constant and $a^{-1}\int_0^a g(x) + p(x) \d x = l^{-1}\int_a^{a+l} p(x) \d x = C$ with $-C \leq g(t) \leq 1-C$.
\end{proof}
}

\end{document}